\documentclass{article}
\usepackage{amsmath, amsthm, amssymb}
\usepackage{algorithm}
\usepackage{algorithmic}

\newtheorem{definition}{Definition}
\newtheorem{statement}{Statement}
\newtheorem{theorem}{Theorem}
\newtheorem{problem}{Problem}

\title{Learnable: Theory vs Applications }
\author{Marina Sapir }
\date{March 2018}

\usepackage{graphicx}
 \usepackage[numbers]{natbib} 

\begin{document}

\maketitle

\begin{abstract}
Two different views on machine learning problem: Applied learning (machine learning with business applications)  and Agnostic PAC learning are formalized and compared here.  I show that, under some conditions, the theory of PAC Learnable provides a way to solve the Applied learning problem. However, the theory requires to have the training sets so large, that it would make the learning practically useless. I suggest to shed some theoretical misconceptions about learning to make the theory more aligned with the needs and experience of practitioners.

\end{abstract}

\section{Introduction}

Machine learning includes a practical side as well as a theoretical side. Practitioners solve real life problems, theoreticians study theory of learning. Practitioners need help answering the questions the life poses. Theoreticians give the  answers. Unfortunately, they are, apparently,  answering different questions, about some-what different subjects. For example, practitioners deal with limited data and deadlines, while the theory talks about  what happens  when  training data increase indefinitely. There appears to be some disconnect here.

More the over, the practitioners often can not formulate their questions exactly,  because  the language of the existing theory  was developed by theoreticians to  study different issues and different situations.

Here I formulate the learning problem as it is encountered in practical applications, pose the related real life  questions and show the answers to these questions which follow from PAC learnable theory. To do it, I express both  questions of practitioners and the results of the PAC learnable theory in terms of problem solving.

\section{Practical point of view}

What is machine learning, practically speaking? Let us look  on the typical situation when such a problem arises.

\subsection{Business learning scenario}

Machine learning is used when a business wants to model a relationship between some observed properties  and a hidden property of   objects of interest.

For example, a business may want to know which patients will be responsive to certain therapy based on their test results, or which engines require certain repairs based on the sensors data, which clients will default on loans and so on. If the business learns to predict the hidden property most of times, it can help with decision making about clients, patients or engines.

Often, this prediction is part of the company's product, and it  is used for decision making  outside of the company. For example, a doctor will use recommendations to prescribe the drug, and a mechanic will use the suggested diagnostics for engine repairs.

 To do the modeling,  the company would have accumulated data about both observed and the hidden properties of a limited set of such objects.  One way to proceed would be to have so many examples that any new example would be the same or very similar to the ones already known. Most of times, it is impossible: there are not enough patients, clients or engines in the world for that. Any way, the cost of gathering huge amounts of data about real life objects is, usually,  prohibitive.  It is exactly what makes application of  machine learning attractive. Machine learning, as any data modeling, is used to compensate for insufficient knowledge which can be gained form experience.

 There is common understanding that the observed features and the hidden property are affected by unknown, random, unaccountable factors. It means, the company has reasonably limited expectations for the accuracy of the prediction.

In a typical scenario, when a business encounters the research problem  requiring prediction of a hidden property, it supplies its machine learning scientists with the precious (but not precise) accumulated data on the observed and hidden properties of the objects, the threshold  of an admissible error rate, as well as a deadline.

\subsection{Definition of the Applied learning problem}

Since classical theory of learning deals with binary classifications, let us give an exact definition of a learning problem for a case of binary classification as it is encountered in real life.

\begin{problem}[Applied Learning  ]
Given: a threshold $ \epsilon \in (0,1)$ and a random  sample of $m$ i.i.d. pairs $S_D = \{( x_i, y_i ), \; x_i \in \chi, \; y_i \in \gamma = \{0, 1\}, i = [m]$\} generated  by unknown probability distribution $D$ on $\chi \times \gamma$; To find:  a function $h: \chi \rightarrow \gamma$ such that
$$ L_D(h) = P_{(x,y) \sim D} [h(x) \neq y]    \le \epsilon $$.

\end{problem}

The given random sample $S$  is called training set.
An instance of the AL problem with the threshold $\epsilon $ and the training set $S$ generated by a distribution $D$ will be denoted by $ (\epsilon, S)_D.$

The elements from $\chi$ are called ``points''; the elements from $\gamma$ ``labels''.  The training set is, usually,  interpreted as past knowledge. The output function $h: \chi \rightarrow \gamma$ is said to ``predict'' labels $y$ given points $x$ of the pairs $(x, y)$.

The function $L_D(h), $ the  probability that the function $h$ incorrectly predicts a label, is called ``generalization error''.

The output $h$ with  generalization error below $\epsilon$ is the ``solution'' of $(\epsilon, S)_D, $ or $h \rhd (\epsilon, S)_D. $

 The procedure which defines a function $\chi \rightarrow \gamma$ given the instance of the problem is called ``learner''.  If, for a learner $\Delta$, $\Delta(S) \rhd (\epsilon, S)_D$ the learner $\Delta$ \textit{succeeds} on the  $(\epsilon, S)_D$, otherwise it \textit{fails} on this instance of the problem. The instance  $(\epsilon, S)_D$ is soluble, if there exists a learner $\Delta$ such that
 $\Delta(S) \rhd (\epsilon, S)_D. $

The critical questions machine learning scientists will need to answer are:

\begin{enumerate}
    \item Is given instance of the problem soluble?
    \item How to find a solution? What may be the best approach for the given data?
    \item How to verify the found function?
\end{enumerate}

\subsection{Main Challenges of Learning}

To understand the main challenges of learning we may look at the relationship between the Applied learning and a decidable computational problem.

 A computational problem is a   function  $f: X \rightarrow Y,  $  defined in a user-friendly language.

For example, ``Given a Boolean formula $x,$ to find  a combination $y$ of values of its variables such that $x(y) = 1$'', or ``Given a weighted graph  and two of its vertices $x = \langle G, A, B \rangle, $ to find a shortest path $y$ between $A, B$ on $G.$''

If the problem is ``decidable'', there is a procedure, which outputs $f(x)$ given $x \in X.$  More the over,  an intended user of the problem's definition can reproduce this procedure. The goal is, usually, to develop an algorithm which does the same computation efficiently.

 In a learning problem, the function $y= f(x)$ is represented by a  sequence  of its inputs and outputs $S = \{(x_i, y_i)\}$, with the similar goal to produce the procedure  which outputs $f(x)$ on each instance $x \in X$.

Since the   function is not defined on its whole domain, there is no way of testing if, for some new input,  the predicted output is ``correct''. More the over, while a computational problem defines a deterministic function, the relationship between inputs $x$ and outputs $y$ in the learning problem is a random function. An input $x$ may have different outputs, each with its own probability.

Summarizing,   there are three major specific challenges presented by learning:

\begin{enumerate}
    \item Uncertainty of  labels. If the distribution $D$ is arbitrary, there is no way to predict the labels even  on the points $X_S = \{x_i, i = [m]\}$ of the  training set.

    \item Unpredictability. There is nothing in the conditions of the problem which would allow one to predict the labels on the points $x \in \chi \setminus X_S,$ even if one believes that the labels on the points of the training set have no uncertainty.

    \item Unspecified Verification.  The problem does not supply a way to confirm that any function is the solution.
\end{enumerate}

The challenges make the Applied learning problem not just ill posed, but somewhat incongruous. Generally speaking, the inputs have no bearing on the output, as in the famous problem by Good Soldier Shvejk (as I recall it): ``the house has 2 floors, there are 6 windows on each floor; what is the name of my grandmother? ''

Each of these challenges is a serious, unavoidable issue in practical applications.

 Take, for example, the uncertainly of labels. Since many factors affecting the labels are unknown and not reflected in the descriptions $x$, the relationship between $x$ and $y$ are not deterministic. The level of uncertainty is impossible to evaluate  if the domain $X$ is infinite: the training set's point are all different, usually.

To appreciate the difficulty the  labels uncertainty brings among other issues, let us consider a vastly simplified Toy problem where two other issues are irrelevant.

\begin{problem}[Toy]
Input: The domain $\chi$ contains $m$ points $\chi = \{x_i\}, i = [m]$. The marginal distribution $D_{\chi}$ is uniform, and for every there exists $q, 0 < q < 0.5 $ such that $x_i,   \; D(x_i, 0) = q$ or $D(x_i, 1) = q.$
Any training set  $S$ has $m$ instances, with all the points of the domain $\chi.$
Output:  a function  $h: \chi \rightarrow \gamma$ such that
$$ L_D(h) = P_{(x,y) \sim D} [h(x) \neq y]    \le \epsilon$$.
\end{problem}

The Toy problem is a specific case of the Applied learning problem.  Each settings $(m, q, \epsilon) $ describe a collection of $2^m$ instances of the problem with different training sets.

In each instance of the Toy problem, the points of a training set cover whole domain $\chi$, so one does not have to worry about unpredictability outside of the training set or the validation.

If $\epsilon < q,$ no solution of the Toy problem is possible.

If $\epsilon > q,$ there is a trivial solution: the function $$h^{\prime}(x) = arg \; max_{(0,1)} (D(x, 0), D(x, 1) ). $$
The function $h^\prime$ has generalization error $q$, which satisfies the goal. The solution does not depend on the training set.

Yet, every learner fails on some training sets.

\begin{statement} There are settings of the Toy problem where the solution exists, but  any learner fails on at least $20\%$ of training sets.
\end{statement}

\begin{proof}
For each instance of the Toy problem, a ``trivial learner'' $t$ is equivalent to its training set:   $h_S^t: h_S^t(x_i) = y_i, i = [m]$, or $h_S^t = S.$

Let us notice that a training set $S$   splits domain $\chi$ on two parts: $\chi  = X_1 \bigcup X_2,$ such that if $x_i \in X_1, \; D(x_i, y_i) = 1 - q $, and if $x_i \in X_2, \; D(x_i, y_i) = q. $ It is convenient to call the points $X_1$ ``straight'', because they have the most likely labels in the training set, and the points in the part $X_2$ will be called ``flipped''.

The probability $z(k, m)$  that the training set has exactly $k > 0$ flipped points out of $m$ can be calculated by the Bernoulli formula: $z(k, m) = C_m^k \cdot p^k \cdot q^{m-k}.  $ The probability $z(0, m)$ that there are no flipped points in the training set is $ t = q^{m}.$ The probability $G$ that there is not more than $\epsilon$ of flipped points in the training set can be calculated by the formula
$$G  = \sum_0^{\epsilon \times m}z(i, m). $$ For  $m = 50, \; q = 0.1, \; \epsilon = 0.12. $ the calculations yield $G = 0.77, G < 0.8. $ So, for a trivial learner, probability of success is less than $80\%. $

If the learner is not trivial, it changes some labels on the training set to produce the hypothesis. Since there is no way to know which points are straight,  it will change labels randomly on some straight points and some flipped. Changing labels on straight points increases error of the hypothesis, changing labels on flipped points decreases error of the hypothesis comparing with the trivial learner. For the settings $m = 50, \; q = 0.1, \; \epsilon = 0.12 $   there expected to be much more straight points (9 times more) than flipped points, the  expected number of changing labels  on straight points will be proportionally higher than the number of changed labels on flipped points. Then the expected number of errors of a hypothesis of the non-trivial learner will be higher than for the trivial learner.

\end{proof}

It is obvious that the issue with Applied learning is that there is too much uncertainty in the problem. On another hand, many practical machine learning problems were solved. It means, there are commonly occurring situations when the problem is soluble.  To find the general approach to the learning problem, the problem needs to be reformulated with additional, natural constrains on input data and on on solutions. One path to reformulate the problem is explored in PAC learnable theory.

\section{Theoretical approach to learning}

The main idea of learning in theory \cite{ShalevLearnabilityStability} is to
minimize generalization errors  to within arbitrary precision based only on a finite training set as the size of  the given training set tends to infinity. (The idea that the training set increases indefinitely sounds very theoretical.)

\subsection{The main concepts and results of PAC learning}

Even though the theory of learning does not have an exact definition of the learning problem, there is the concept of ``learnable'', which I will use to formulate the learning problem.

\begin{definition}[Agnostic PAC Learnable]
A function class $H$ is agnostic PAC learnable if there exists a function $m_H: (0,1)^2 \rightarrow N $ and a learner $\Delta_H$ with the following property: For every $\epsilon, \delta \in (0,1),$ and for every distribution $D$ over $\chi \times \gamma $, when running the learner $\Delta_H$ on $m \ge m_H(\epsilon, \delta)$ i.i.d. instances generated by $D$, it outputs a function $h \in H$ such that, with probability of at least $1 - \delta$ (over the choice of the training examples),
$$ L_D(h) \le \min_{h^\prime \in H} L_D(h^\prime) + \epsilon.$$

\end{definition}

The concept implies that ``learning'' means solving the next problem:

\begin{problem} [Agnostic PAC Learning] Given are: class of functions $H: \chi \rightarrow \gamma,$ two numbers $\epsilon, \delta \in (0,1)$. To find: a learner  which,  for every large enough  randomly generated sequence of i.i.d. pairs $S \in  (\chi \times \gamma)^*$ from an arbitrary random distribution $D$ on $\chi \times \gamma,$   finds a function $h \in H$ such that $$ L_D(h) \le \min_{h^\prime \in H} L_D(h^\prime) + \epsilon$$ with probability of at least $1 - \delta.$

\end{problem}

The learner, which solves the Agnostic PAC learning problem for the class $H$ is called the successful learner of the class $H$. The number $ m_H(\epsilon, \delta)$  of instances sufficient for probably approximately correct learning for the given $\epsilon, \delta$ is the property of the  Agnostic PAC learnable class $H$. It can be  called a $(\epsilon, \delta)$-saturation point for the class $H.$

The differences between the Agnostic PAC learning problem and the Applied learning problem are summarized in the Table ~\ref{table:problems}.

\begin{table}
\caption{Applied vs Agnostic PAC Learning}

\begin{tabular}{| c | c | c |}
  \hline
    & Applied Learning & Agnostic PAC learning \\ \hline
  Training set & Fixed, finite  & Random generator, unlimited\\ \hline
  Distribution & Fixed distribution & Any distribution \\ \hline
  Hypothesis source & No restriction & Class $H$ \\ \hline
  Desired error &  $ \le \epsilon$ & $ \le min\; L_D(h^\prime) + \epsilon$ \\ \hline
  Chance of bigger error & Not specified & $\delta$ \\
  \hline
\end{tabular}
 \label{table:problems}

\end{table}

In general, the problems are not comparable.

For example, a learner finding solution $h \in H$ on the instance $(\epsilon, S)_D $ of Applied learning problem may not be a successful PAC learner for the class $H$ if it does not work on many other distributions, regardless how large the training set is.

On another hand, a successful PAC learner for a class $H$ may not produce a solution for the instance $(\epsilon, S)_D $ of Applied learning problem even if it has a  solution in the class $H,$ if the size of the training set $S$ is below the $(\epsilon, \delta)$-saturation point for the given class $H$  for a  small enough $\delta.$

The main idea of Agnostic PAC learning  is to identify the classes of functions, where large enough training sets allow to approximate the function on whole domain reliably, regardless of distribution.  The ``uniform convergence'' property  defines such classes.

\begin{definition}[Uniform Convergence] A hypothesis class $H$ has the uniform convergence property if for every $\epsilon, \delta \in (0, 1)$ there exists $M$ such that for every probability distribution $D$  over $\chi \times  \gamma$, if the sequence of i.i.d. pairs $S \in (\chi \times \gamma)^*, \; S \sim D$ has the length over $M$ then then with probability at least $1- \delta$  $S$ satisfies the condition:  for every $h \in H$
$$ | L_S(h) - L_D(h) | \leq \epsilon  .$$
\end{definition}

In the definition, the threshold $M$ depends only on $\epsilon, \delta$, and the class $H.$ The definition implies that the class $H$ with uniform convergence can be characterized by its function $M = m^u_H(\epsilon, \delta).$

 The  results of PAC Learnable theory are neatly summarized in The Fundamental Theorem of Statistical Learning (\cite{Shalev}, p48).

\begin{theorem}[Fundamental Theorem]
The next conditions are equivalent:
\begin{itemize}
\item $H$ is agnostic PAC learnable;
\item $H$ has property of uniform convergence.
\item Any empiric risk minimizer on $H$ is a successful PAC  learner for this class;
\item $H$ has finite VC-dimension.
\end{itemize}
\end{theorem}

The empirical risk in the theorem refers to average of errors of a hypothesis on the training set: $L(h, S)= \sum( |h(x_i) - y_i| ) / m . $

The term ``empiric risk minimizer on $H$''  in the theorem refers to a procedure which, given a training set $S$   finds a function  $h \in H$ minimizing the empiric risk. It is convenient to denote $ERM_H(S)$ a hypothesis obtained by an empiric risk minimizer in a class $H$ on the training set $S.$

\subsection{Implications for Applied Learning}

The next theorem shows how and when this theory can help to solve the Applied learning problem.

\begin{theorem}[From Theory to Applications] \label{Yields}

Suppose, an instance  $(\epsilon, S)_D,$ of the Applied learning problem satisfies
   $$|S| >   max( m_H^u(\epsilon /2 , \delta), m_H(\epsilon /2 , \delta))$$  for a class $H$ with finite VC-dimension and $0 < \delta < 1.$

Then, for $h = ERM_H(S),$  if $L_S( h ) \le \epsilon / 2$, $h $ is the solution for $(\epsilon, S)_D$ with probability at least $1 - \delta$.  If $L_S(h) > 2 \epsilon $ the instance does not have a solution within class $H.$
\end{theorem}

\begin{proof}
According to the Fundamental theorem, the class $H$ has uniform convergence and is agnostic PAC learnable.

Since $|S| >  m_H^u(\epsilon /2, \delta), $   the uniform convergence implies that for any hypothesis $h \in H,$ $ | L_S(h) - L_D(h) | \leq \epsilon /2. $ If $L_S(h) \le \epsilon / 2$, then $ L_D(h) < \epsilon$  with probability at least $1 - \delta.$ It means, $h$ is a solution of the $(\epsilon, S)_D$ with probability at least $1 - \delta.$

Suppose, $L_S(h) >  2 \epsilon.$ By the property of uniform convergence  then $L_D(h)  >  1.5 \epsilon.$
By the property of PAC learnable $L_D(h) - min_H( L_D(h') ) < \epsilon / 2.$ Then $min_H(L_D(h')) > \epsilon.$ Therefore, the instance $(\epsilon, S_D)$ does not have a solution in $H.$
\end{proof}

In the conditions of the Theorem \ref{Yields}, we can answer all the critical questions in relation to the class $H$:

\begin{enumerate}
    \item Does the instance $(\epsilon, S)_D$ of the problem have a solution in class $H$? It does, if the empiric risk minimizer has the empiric risk below $\epsilon / 2$. It does not,  if the empiric risk is above $2 \epsilon.$
    \item What is the best approach to solve the problem? Applying any empiric risk minimizer.
    \item How to verify the solution? There is not need to verify it  if empiric risk is below $\epsilon /2$ and the probability of not finding a solution  $\delta$ is small enough.
\end{enumerate}

But how often do Applied problems satisfy the requirement on the size of the training set  in the Theorem ~\ref{Yields}?

There is the estimate of the saturation points  in \cite{Shalev}, p 341.  The table ~\ref{table:size} shows the estimates of saturation point depending on the thresholds $\epsilon, \delta$ and the VC dimension of $H.$

\begin{table}
\centering
\caption{Estimates of $m_H$}

\begin{tabular}{| c | c | c | c |}
  \hline
     $\epsilon$ & $\delta $ &VCdim(H) & $m_H$ \\ \hline
 0.2 & 0.05 & 5   & $140\;672$ \\ \hline
  0.2 & 0.05 & 10 & $290\;6826$ \\ \hline
  0.2 & 0.05 & 30 & $921\;275$ \\ \hline

0.1 & 0.05 & 5   & $651\;412$ \\ \hline
  0.1 & 0.05 & 10 & $1\;340\;176$ \\ \hline
  0.1 & 0.05 & 30 & $4\;217\;438$ \\ \hline

\end{tabular}

 \label{table:size}
\end{table}

To anybody familiar with real life applied problems these estimates look very wrong, excessive. For comparison, the table ~\ref{table:UCI} shows sizes of the classification problems with not more than 4 classes among the most popular problems in UCI Machine Learning Repository.

\begin{table}
\centering
\caption{UCI Repository}

\begin{tabular}{| c | c | c| c |}
  \hline
 Data & Features & Records & Classes\\ \hline
 Iris  & 4  & 150 & 3 \\ \hline
 Adult & 14 & 48842 & 2 \\ \hline
 Car Evaluation & 6 & 1728 & 4 \\ \hline
Breast Cancer & 32 & 569 & 2 \\ \hline
Heart Disease & 75 & 303 & 2 \\ \hline
Bank Marketing & 20 & 41188 & 2 \\ \hline
\end{tabular}

 \label{table:UCI}
\end{table}

The comparison between the tables ~\ref{table:size} and ~\ref{table:UCI} gives a hint  on why the theory is never used in the applications. The theory can be applied only in the cases, when the function we are trying to approximate is assumed to have low variability and known on so many points, that we already know  (almost) everything there is to know. The issue is, in this case there is no need in machine learning.

\section{Conclusions}

The learning problem presents formidable challenges with too much uncertainly,  unknowable and unverifiable. The only way to find some solution is to narrow the search, impose some restrictions on solutions.

The theory of PAC learnable explores searching for a solution (1) within classes of function with limited complexity (finite VC dimension) , and (2) with huge training sets.

Imposing these requirements reduces the uncertainty of a problem to such a degree that the problem can be either solved  with a simple function approximation algorithm, or it can be shown not to have  a solution within the given class of functions.

However, these requirements are virtually never satisfied in applications, because they would contradict the goal of learning: to compensate for  deficiency of accumulated data.

The theory does not help to solve the real life problems, and it does not explain, why so many machine learning tasks with relatively small training sets are being solved successfully. It does not explain why so many algorithms developed specifically for machine learning work, when empiric risk minimization does not.

Perhaps, there shall  be another theory which allies better with machine learning practitioners' experience and needs.

And here is what I would expect from this theory of practical machine learning :
\begin{enumerate}
\item The concept of learning process shall not involve indefinite increase in the training set size, nor  an indefinite improvements in accuracy.  Neither is possible or desirable in applications. Learning is one time event with the specific distribution, training set, and desired accuracy.
\item The concept of  learning shall not involve the idea that we need to evaluate generalization error based on the error on the training set. In real life, people use the set aside data for that.
\item The concept of a learning process shall include the set aside data as its integral part.
\item The study of learning shall find desirable properties of distributions which make real life learning possible.
\item Perhaps, the study of learning can find some desirable properties of the training sets, besides their size.
\item The popular machine learning algorithms (such as Naive Bayes, SVM, kNN)  can be used as case studies to demonstrate that the tasks with desirable properties can be solved by  most of them.
\end{enumerate}

\bibliographystyle{authordate1}
\bibliography{Learnable}

\end{document}